\documentclass[11pt]{article}

\usepackage{amsmath}
\usepackage{amsthm}
\usepackage{amssymb}
\usepackage{algorithm}
\usepackage{subfig}
\usepackage{color}
\usepackage[english]{babel}
\usepackage{graphicx}
\usepackage{grffile}
\usepackage{wrapfig,epsfig}
\usepackage{epstopdf}
\usepackage{url}
\usepackage{color}
\usepackage{epstopdf}
\usepackage{algpseudocode}
\usepackage{bbm}
\usepackage{comment}
\usepackage{dsfont}

\usepackage[utf8]{inputenc} 
\usepackage[T1]{fontenc}    
\usepackage{booktabs}       
\usepackage{amsfonts}       
\usepackage{nicefrac}       
\usepackage{microtype}      

\usepackage{tikz}
\usepackage{hyperref}  
\hypersetup{colorlinks=true,citecolor=blue,linkcolor=blue} 
\usetikzlibrary{arrows}
\usepackage[margin=1in]{geometry}
\graphicspath{{./figs/}}

\newtheorem{theorem}{Theorem}[section]
\newtheorem{lemma}[theorem]{Lemma}

\newtheorem{proposition}[theorem]{Proposition}

\newtheorem{assumption}[theorem]{Assumption}

\newtheorem{claim}[theorem]{Claim}

\newcommand{\wh}{\widehat}
\newcommand{\wt}{\widetilde}

\newcommand{\N}{\mathcal{N}}

\newcommand{\R}{\mathbb{R}}
\renewcommand{\P}{\mathbb{P}}
\newcommand{\I}{{\bf 1}}

\renewcommand{\varepsilon}{\epsilon}
\renewcommand{\tilde}{\wt}
\renewcommand{\hat}{\wh}

\DeclareMathOperator*{\E}{{\mathbb{E}}}

\DeclareMathOperator{\poly}{poly}

\makeatletter
\newcommand*{\RN}[1]{\expandafter\@slowromancap\romannumeral #1@}
\makeatother

\usepackage{lineno}

\begin{document}
\date{}
\title{Efficient Model-free Reinforcement Learning  in Metric Spaces}
\author{
Zhao Song\thanks{zhaosong@uw.edu. University of Washington. Part of the work was done while being hosted by Yin Tat Lee.}
\and
Wen Sun\thanks{wensun@andrew.cmu.edu. Carnegie Mellon University.}
}


\begin{titlepage}
 \maketitle
  \begin{abstract}

Model-free Reinforcement Learning (RL) algorithms such as Q-learning [Watkins, Dayan'92]
have been widely used in practice and can achieve human level performance in applications such as video games [Mnih et al.'15]. 
 Recently, equipped with the idea of optimism in the face of uncertainty, Q-learning algorithms [Jin, Allen-Zhu, Bubeck, Jordan'18] 
can be proven to be sample efficient for \emph{discrete tabular} Markov Decision Processes (MDPs) which have finite number of states and actions. In this work, we present an efficient model-free Q-learning based algorithm in MDPs with a natural metric on the state-action space---hence extending efficient model-free Q-learning algorithms to continuous state-action space. Compared to previous model-based RL algorithms for metric spaces [Kakade, Kearns, Langford'03], 
our algorithm does not require access to a black-box planning oracle.



  \end{abstract}
 \thispagestyle{empty}
 \end{titlepage}

\section{Introduction}

In Reinforcement Learning (RL), there are two families of algorithms: model-free RL algorithms (e.g., Q-learning \cite{watkins1992q}) and model-based RL ones (e.g., \cite{azar2017minimax}). While there are comparisons between these two families in terms of sample and computation efficiency (e.g., \cite{kearns1999finite,sun2018model,tu2018gap}), model-free methods are often popular in practice due to its simplicity for not requiring planning oracles and the surprising practical efficiency (e.g., policy gradient \cite{sutton2000policy,schulman2015trust} and Q-learning \cite{watkins1992q,mnih2015human}).

Recently, in discrete tabular MDPs settings (i.e., MDPs with finite number of states and actions), \cite{jabj18} show that equipped with optimism, classic Q-learning algorithms can balance the exploration and exploitation trade-off to  achieve regret bounds that are comparable  to the regret bounds of previous known model-based efficient RL algorithms (e.g., \cite{azar2017minimax}), while enjoying a better computational complexity and space complexity. Specifically, the algorithms in \cite{jabj18} do not require access to planning oracles anymore. However the algorithms in \cite{jabj18} relies on the discrete nature of the MDPs and cannot be directly used in continuous state-action space. 
 
In this work, we examine the problem of trading exploration and exploitation in MDPs where we have continuous state-action space with a natural metric,  \emph{under the model-free learning framework}. Previous works \cite{kakade2003exploration,ortner2012online,lakshmanan2015improved} considered efficient model-based RL in metric space where the proposed algorithm requires access to a planning oracle which itself could be a NP-hard problem in the continuous setting. The main assumption in our work is the property that ``nearby" state-action pairs have ``similar" optimal values. Such condition is common and indeed is a more general version of the smoothness assumptions on the transition dynamics and reward functions that were used in \cite{kakade2003exploration,ortner2012online,lakshmanan2015improved} (i.e., smoothness in transition dynamics and reward function implies smoothness in optimal value functions, but not the other way around).

We formalize these natural and general assumptions, and prove that even under the \emph{model-free setting}, they are sufficient for achieving near-optimal policies in an amount of time depending on the metric resolution, but not on the size of the state-action space. More specifically, we propose \emph{Net-based Q-learning} (\textsc{NbQl}), a Q-learning like algorithm under the principle of optimism under the face of uncertainty, which can learn a near-optimal policy with regret and space scaling with respect to the covering number---a natural and standard notion of the resolution under the metric. \textsc{NbQl} encourages efficient exploration by an approximate count-based strategy. Different from count-based exploration in tabular MDPs (e.g., \cite{brafman2002r,jabj18}) where one maintains visit counts for \emph{every} state-action pair, in large or continuous state-action space, we will not be able to afford to do so both computation-wise and space-wise. Instead, \textsc{NbQl} only maintains visitation counts over a subset of state-action pairs, which provides generalization to the entire state-action space via the underlying Lipschitz continuity assumption. Note that such kind of approximate count-based exploration strategy was empirically studied and could achieve state-of-art performance on several RL benchmarks \cite{tang2017exploration}. Hence our work can be regarded as providing a theoretical justification for such approximate count-based exploration strategy. \textsc{NbQl} works under the episodic finite horizon setting, without assuming the existence of a generative model (e.g., generative models are often used to alleviate the challenges of exploration \cite{kearns1999finite,sidford2018near,sidford2018variance}), without building a transition model, and without  requiring a near-optimal planning oracle.  Our algorithm is simple and similar to classic Q-learning. To the best of our knowledge, this is the first provably efficient model-free Q-learning algorithm in metric spaces. 

\section{Preliminaries}

We consider episodic finite horizon MDP $(\P, {\cal S}, {\cal A}, \mathbf{r}, H)$ where $\cal S$ and $\cal A$ are the state and action space; $\P$ is the transition kernel such that $\P_h(\cdot | x,a)\in\Delta(\cal S)$ gives the state distribution if action $a$ is taken at state $s$ at step $h$; ${\cal S}$, $\mathbf{r}$ is the reward function such that $r_h(x,a)\in [0,1]$ is the reward of taking action $a$ at state $s$ at time step $h$; $H$ is finite horizon. We define policy $\pi: {\cal S}\to{\cal A}$ as mapping that maps from states to actions. Given a policy $\pi$, we define value function as
\begin{align*}
V^{\pi}_h(x) = \E \left[ \sum_{i=h}^H r_i(x_i,a_i) ~\Big|~ x_h = x, a_i = \pi(x_i) \right],
\end{align*}where the expectation is taken with respect to the randomness of the MDP. We also define state-action value function as
\begin{align*}
Q^{\pi}_h(x,a) = r_h(x,a) + \E_{x'\sim \P_h(\cdot|x,a)}[V_{h+1}^{\pi}(x')],
\end{align*} where for notation simplicity we define $V^{\pi}_{H+1}(x) = 0$ for any $\pi$ and $x\in\cal S$. The optimal $Q$ function $Q^*$ satisfies the Bellman optimality $Q_h^*(x,a) = r_h(x,a) + \E_{x'\sim \P_h(\cdot|x,a)}[\max_{a\in{\cal A}}Q^*_{h+1}(x,a)]$, and the optimal policy $\pi^*$ is induced from $Q^*$ as $\pi^*_h(x) = \arg\max_{a\in{\cal A}}Q^*_h(x,a)$.

Following the assumption in \cite{kakade2003exploration}, we assume that there is a distance metric $D:({\cal S}\times{\cal A})^2 \to \R^+$, such that $d\left((x,a), (x',a')\right)$ measures the distance between two state-action pairs, and $D((x,a), (x',a')) = 0$ if and only if $x=x',a=a'$,  and $D((x,a),(x',a')) = D((x',a'),(x,a))$ (i.e., symmetric). Given $\cal S\times\cal A$ and the metric $D$, we define an $\epsilon$-net of the metric space $({\cal S}\times{\cal A}, D)$ as ${\cal N}(({\cal S\times{\cal A}}),\epsilon, D) \subset {\cal S\times\cal A}$, such that for any $(x,a)\in\cal S\times\cal A$, there exists a $(x',a')\in{\cal N}({\cal S}\times{\cal A}, \epsilon, d)$, such that $D((x,a),(x',a')) \leq \epsilon$. Below we use $\cal{N}_{\epsilon}$ to denote the $\epsilon$-net that has the smallest size, and $|\N_{\epsilon}|$ is defined as the \emph{covering number}.
We define \emph{covering dimension} $d \triangleq \inf_{d>0}\{ |\N_{\epsilon}| \leq \epsilon^{-d}, \forall \epsilon>0 \}$. We refer readers to \cite{clarkson2006building,shalev2014understanding} for details of covering numbers and covering dimensions. 

Below, we show the main assumption in our work---the Lipschitz continuous assumption on the optimal Q function:
\begin{assumption}[Lipschitz Continuous $Q^*$]\label{assum:lipschitz}
We assume that for any $h\in [H]$, $Q^*_h$ is Lipschitz continuous as:
\begin{align*}
    |Q_h^*(x,a) - Q_h^*(x',a')| \leq D\left((x,a), (x',a')\right), \forall (x,a,x',a').
\end{align*}
\end{assumption} 
The above assumption captures the settings where nearby state-action pairs have similar $Q^*$ values.  Previous works on RL in metric spaces often assume Lipschitz continuous in transition kernel and reward function, which in turn actually implies Lipschitz continuous in $Q^*$. 
\begin{proposition}
\label{prop:lipschitz}
If we have Lipschitz continuous transition kernel and reward function, i.e., 
\begin{align*}
    \|\P_h(\cdot|x,a) - \P_h(\cdot|x',a')\|_1 \leq & ~ D((x,a),(x',a')), \\
    |r_h(x,a) - r_h(x',a')| \leq & ~ D((x,a),(x',a')) 
\end{align*} for all $(h,x,a,x',a')$, then we have that $Q^*$ is also Lipschitz continuous:
\begin{align*}
  |Q_h^*(x,a) - Q_h^*(x',a')| \leq  (H-h+1) \cdot D((x,a), (x',a')).
\end{align*}
\end{proposition}
For completeness, we include the proof of the above proposition in Appendix~\ref{app:lipschitz}. Hence our assumption~\ref{assum:lipschitz} above is weaker than the assumptions used in previous related works.

We measure our algorithm's sample efficiency via \emph{regret}. At the $k$-th episode, the initial state $x_1^k$ is revealed to the learner ($x_1^k$ could be chosen by an adversary), the learner picks a policy $\pi^k$, and execute the policy $\pi^k$ for $H$ steps to reach the end of the episode. The cumulative regret of the learner is defined as
\begin{align*}
    \textsc{Regret} = \sum_{k=1}^K V^*_1(x_1^k) - V_1^{\pi_k}(x_1^k).
\end{align*}

We will use Azuma-Hoeffding inequality to construct confidence bound of $Q^*$. For completeness, we state Azuma-Hoeffding inequality below. 
\begin{lemma}[Azuma-Hoeffding inequality]
Suppose $\{ X_k : k = 0,1,2,3, \cdots \}$ is a martingale and $| X_k - X_{k-1} | < c_k $, almost surely. Then for all positive integers $N$ and all positive reals $t$,
\begin{align*}
    \Pr[ | X_N - X_0 | \geq t ] \leq 2 \exp \left( \frac{ -t^2 }{ 2 \sum_{k=1}^N c_k^2 } \right).
\end{align*}
\end{lemma}

{\bf Notations.}
For real numbers $a,b, \epsilon$, we use $a = b \pm \epsilon$ to denote that $ a \in [b-\epsilon,b+\epsilon]$. For an integer $K$, we use $[K]$ to denote the set $\{1,2,\dots, K\}$. We use ${\bf 1}$ to denote the indicator function such that ${\bf 1}[f] = 1$ if $f$ holds and ${\bf 1}[f] = 0$ otherwise. For any function $f$, we define $\wt{O}(f)$ to be $f\cdot \log^{O(1)}(f)$.

\section{Algorithm}

Our algorithm takes an $\epsilon$-net $\N_{\epsilon}$ as input. For any $(x,a)\in{\cal S}\times{\cal A}$, we define $\phi: {\cal S}\times{\cal A}\to \N_{\epsilon}$ as the mapping that maps $(x,a)$ to the closest point in the net, i.e., \begin{align*}
\phi(x,a) \triangleq \arg\min_{(x',a')\in \N_{\epsilon}}D((x,a),(x',a')).
\end{align*} Since $\N_{\epsilon}$ is an $\epsilon$-net, then we have $D((x,a),\phi(x,a))\leq \epsilon$, which in turn implies that $|Q^*_h(x,a) - Q_h^*(\phi(x,a))| \leq \epsilon$ via assumption~\ref{assum:lipschitz}. This intuitively means that when $\epsilon$ is small, as long as we can accurately estimate the optimal value $Q^*$ of the points in $\N_{\epsilon}$, we will be able to achieve a uniformly accurate estimation of $Q^*$ at any state-action pair.



\begin{algorithm}[t!]\caption{}\label{alg:main}
\begin{algorithmic}[1]
\Procedure{\textsc{Net-based Q-learning} }{${\cal S}, {\cal A},\N_{\epsilon}, H, K $} \Comment{Theorem~\ref{thm:main}}
\For{$h = 1 \to H$}
    \For{$(x,a) \in {\cal N}_{\epsilon}$}
        \State $Q_h(x,a) \leftarrow H$ 
        \State $n_h(x,a) \leftarrow 0$
    \EndFor
\EndFor

\For{episode $k = 1 \to K$}
	\State {\bf Receive} $x_1$ 
	\For{step $h = 1 \to H$}
		\State {\bf Take} action $a_h \leftarrow \arg\max_{a'} Q_h(\phi(x_h, a'))$ \label{line:take_action}
		\State {\bf Receive} state $x_{h+1}$  
		\State $n_h(\phi(x_h, a_h)) \leftarrow n_h(\phi(x_h, a_h)) + 1$ 
		\State $t \leftarrow n_h(\phi(x_h, a_h))$
		\State $b_t \leftarrow c \sqrt{H^3 \gamma /t }$ \label{line:bonus}
		\State $Q_h(\phi(x_h, a_h)) \leftarrow (1-\alpha_t) \cdot Q_h(\phi(x_h,a_h)) + \alpha_t \cdot ( r_h(x_h,a_h) + V_{h+1}(x_{h+1}) + b_t ) $ \label{line:update_Q}
		\State $V_h(x_h) \leftarrow \min \{ H , \max_{a' \in {\cal A} } Q_h(\phi(x_h , a') ) \}$ \label{line:update_V}
	\EndFor
\EndFor
\EndProcedure
\end{algorithmic}
\label{alg:NBQL}
\end{algorithm}

Algorithm~\ref{alg:NBQL}---Net-based Q-learning (\textsc{NbQl}) implements the above intuition. Alg.~\ref{alg:NBQL} maintains two tables, of which the size is $|\N_{\epsilon}|\times |\N_{\epsilon}|$---hence with space quadratic with respect to the covering number. For any pair of state-action $(x',a')\in {\cal N}_{\epsilon}$, we maintain an estimation $Q_h$ of $Q^*_h(x',a')$, and also maintain a counter $n_h(x',a')$ which increments every time a state-action pair mapped to $(x',a')$ via $\phi$ is visited at time step $h$. Formally, we increase the counter corresponding to $(x',a')\in\N_{\epsilon}$  by one whenever the algorithm encounters a state-action pair $(x,a)$ such that $\phi(x,a) = (x',a')$. The counter will be used to construct bonus (Line~\ref{line:bonus}) for encouraging exploring less frequently visited regions (i.e., regions covered by points in $\N_{\epsilon}$ that have low counts). For large or continuous state-action space, performing such approximate counting provides generalization to unseen states. Such idea was empirically studied in \cite{tang2017exploration} where the proposed algorithm maintains counts over a subset of state-action pairs defined by a mapping $\phi$ (in \cite{tang2017exploration}, $\phi$ is some Hashing function such as SimHash \cite{charikar2002similarity}). Our work provides a theoretical justification for such approximate count-based exploration strategy in continuous state-action space.


To estimate $Q^*_h(x',a')$ for any $(x',a')\in\N_{\epsilon}$, we maintain a $Q$-table $Q_h$ whose entries corresponding to points in $\N_{\epsilon}$, and we use a Q-learning like update as shown in Line~\ref{line:update_Q}. Given $Q_h$, the policy induced by $Q_h$ is $\arg\max_{a\in{\cal A}}Q_h(x,a)$ at any $x\in\cal S$ (Line~\ref{line:take_action}). Given $Q_h$ defined over $\N_{\epsilon}$, we define $V_h(x)$ for any $x\in{\cal S}$ as $V_h(x) = \min\{H, \max_{a\in{\cal A}}Q_h(\phi(x,a))\}$ (Line~\ref{line:update_V}), where we take min since the optimal value $V_h^*(x)$ is bounded by $H$ always. Note that $V_h$ is defined over the entire state space $\cal S$. Though we explicitly write down the form of $V_h$ in Line~\ref{line:update_V},  we emphasize here that we never need to explicitly construct or maintain $V_h$ in Alg.~\ref{alg:NBQL}. Whenever we need to query the value $V_h(x)$ (i.e., in Line~\ref{line:update_Q}), we can use the expression in Line~\ref{line:update_V}.

Regarding computation, first note that Alg.~\ref{alg:NBQL} does not require access to a black-box planning oracle. But we do assume an oracle to compute $\max_{a} Q_h(\phi(x,a))$ and query $\arg\max_a Q_h(\phi(x,a))$ (break tie arbitrarily), which are used in Line~\ref{line:update_V} and Line~\ref{line:take_action}, respectively. Note that the computation time of $\max_{a} Q_h(\phi(x,a))$ for any pair $(x,a)$ is $O(|\N_{\epsilon}|)$---the covering number, since at most we just need to scan through all $(x',a')\in\N_{\epsilon}$ with $x' = x$. Building the optimal $\epsilon$-net in practice is intractable, but one can use the greedy approach to build an $\epsilon$-net: choose $(x^1, a^1)$ arbitrarily from $\cal S \times\cal A$; choose $(x^2, a^2)$ which is at least $\epsilon$ away from $(x^1,a^1)$; choose $(x^3, a^3)$ that is at least $\epsilon$ away from $(x^i,a^i)$ for $i\in [2]$, and so on. Under the assumption $\cal S\times\cal A$ is compact, this procedure returns an $\epsilon$-net whose size is no larger than $|\N_{\epsilon/2}|$---the size of the optimal $\epsilon/2$-net.




\section{Analysis}

In this section, we provide the analysis for $\textsc{NbQl}$. Specifically, we show that the regret bound of $\textsc{NbQl}$ scales as follows. 
\begin{theorem}[main result]\label{thm:main}
There exists an absolute constant $c > 0$ such that, for any $p \in (0,1)$, $\epsilon \in (0,1)$,  with probability at least $1-p$, \textsc{NbQl} (Alg.~\ref{alg:NBQL}) achieves regret at most $O( \sqrt{ H^4 N T \gamma } + \epsilon T )$ where $\gamma = \log ( N T / p )$ and $N$ being the size of the $\epsilon$-net. When $N \leq \epsilon^{-d}$ with $d$ being the covering dimension of $\cal S\times\cal A$, further optimizing $\epsilon$ (set $\epsilon = T^{-1/(d+2)}$) gives the regret:
\begin{align*}
\tilde{O}(H^2 T^{\frac{1+d}{2+d}}).
\end{align*}
\end{theorem}

Note that here $d$ is the covering dimension of $\cal S \times \cal A$. For special case where $d = 1$, we can see that the dependency of $T$ becomes $T^{2/3}$, which matches to the regret bound of model-based approaches \cite{lakshmanan2015improved} under the assumptions of the transition kernel and reward function being Lipschitz continuous. However, comparing to the model-based algorithm from \cite{lakshmanan2015improved}, our algorithm does not require a planning oracle.\footnote{Note that efficient model-based algorithms in \cite{ortner2012online,lakshmanan2015improved} actually need an optimistic planning oracle to choose the most optimistic model from a set of models, which is even a stronger assumption than the assumption of having access to a planning oracle.}   Also the dependency on $T$ cannot be improved in the worst case due to the fact that the lower bound for Lipschitz Multi-armed Bandit scales $\Omega(T^{(d+1)/(d+2)})$ (e.g., Theorem 4.12 in \cite{slivkins2019introduction}).


To prove the above theorem, we provide several useful lemmas below. Throughout this section, we will use $x_h^{k},a_h^k$ to represent the state and action generated at time step $h$ at the $k$-th episode; $Q_h^k$ being the Q-table over the net at time step $h$ at the beginning of the $k$-th episode, and $V_h^k$ is defined using $Q_h^k$ via Line~\ref{line:update_V}; policy $\pi_k$ at the $k$-th episode is induced from $Q_h^k$ as $\pi_k(x) = \arg\max_{a\in{\cal A}}Q_h^k(\phi(x,a))$ at time step $h$. The learning rate $\alpha_t$ is set to be $\alpha_t = (H+1)/(H+t)$ which is the same as the learning rate used \cite{jabj18}. We denote $\alpha_t^0 = \prod_{j=1}^t (1-\alpha_j)$ and $\alpha_t^i = a_i\prod_{j=i+1}^t (1-\alpha_j)$. For any $(x,a,h,k)$, $n_{h,k}(\phi(x,a))$ records the total number of times that $\phi(x,a)\in\N_{\epsilon}$ has been visited at the beginning of the $k$-th episode. 

First, we provide a generalization of Equation 4.3 in \cite{jabj18} below. 

\begin{lemma}\label{lemma:q_formula} 
At any $(x,a, h, k)\in {\cal S}\times {\cal A}\times [H]\times [K]$, let $t = n_{h,k} = n_{h,k}(\phi(x,a))$, and suppose $\phi(x,a)$ was previously encountered at step $h$ of episodes $k_1, k_2, ..., k_{t} < k$, i.e., $\phi(x_h^{k_i} , a_h^{k_i}) = \phi(x,a)$, $\forall i \in [n_{h,k}]$. By the update rule of $Q$, we have:
\begin{align*}
    Q_h^k(\phi(x,a)) = \alpha_{ t }^0 \cdot H + \sum_{i=1}^{ t } \alpha_t^i \cdot \left( r_h(x_h^{k_i},a_h^{k_i}) + V_{h+1}^{k_i}(x_{h+1}^{k_i}) +b_i \right).
\end{align*}
\end{lemma}

The above lemma shows that the $Q_h^k(\phi(x,a))$ is maintained by aggregating the information of the state-action pairs encountered so far whose nearest-neighbor is $\phi(x,a)$ (i.e., $\phi(x_h^{k_i},a_h^{k_i}) = \phi(x,a)$ for all $i \in [n_h^k]$). We defer the proof of the above lemma to Appendix~\ref{app:q_formula}.


Throughout the learning process, we hope that our estimation $Q_h^k(\phi(x,a))$ will get closer to the optimal value $Q_h^*(\phi(x,a))$ for any $(x,a,h)$, as $k$ increases. The following lemma measures the difference between $Q_h^k$ and $Q_h^*$ at the points in the net $\N_{\epsilon}$.

\begin{lemma}
For any $(x,a,h) \in {\cal S} \times {\cal A} \times [H]$ and episode $k \in [K]$, let $t = n_{h,k} (\phi(x,a))$ and suppose $\phi(x,a)$ was previously encountered at step $h$ of episodes $k_1, \cdots, k_t < k$. Then
\begin{align*}
( Q_h^k - Q_h^* ) (\phi(x,a)) &  = \alpha_t^0 \cdot ( H - Q_h^*(\phi(x,a)) ) \\
& ~ + \sum_{i=1}^t \alpha_t^i \cdot \left( ( V_{h+1}^{k_i} - V_{h+1}^{*} ) ( x^{k_i}_{h+1} ) + [ ( \wh{\P}_h^{k_i} - \P_h ) V_{h+1}^* ] (x_h^{k_i},a_h^{k_i}) + b_i \pm \epsilon \right).
\end{align*}
\label{lemma:Q_progress}
\end{lemma}


The next step is to leverage the above lemma and show that with high probability, $Q_h^k(\phi(x,a))$ is approximately an optimistic estimation of $Q_h^*(\phi(x,a))$. To do so, we first observe that the sequence $\{[(\hat{\P}_h^{k_i} - \P_h)V_{h+1}^*](x_h^{k_i},a_h^{k_i})\}_{i=1}^t$ is a Martingale difference sequence, which allows us to use Azuma-Hoeffding inequality to bound the absolute value of the sum of the sequence with high probability. So as long as we set the bonus $\sum_i b_i$ to be large enough, we can guarantee approximate optimism. The proof will also use induction starting from showing optimism at time step $H$, and all the way to $h = 1$. We formally state the upper bound and the lower bound of $(Q_h^k - Q_h^*)(\phi(x,a))$ in the following lemma and defer its proof to Appendix~\ref{app:upper_lower}.

\begin{lemma}\label{lem:bound_on_Q_k_minus_Q_*}
There exists an absolute constant $c > 0$ such that, for any $p \in (0,1)$, with $\gamma = \log\left(NT/p\right)$, letting $b_t = c \sqrt{ H^3 \gamma /t }$, we have $\beta_t = 2 \sum_{i=1}^t a_t^i b_i \leq 4 c \sqrt{H^3 \gamma /t}$ and, with probability at least $1-p$, we have that for all $(x,a,h,k) \in {\cal S} \times {\cal A} \times [H] \times [K]$:
\begin{align}
\mathrm{Upper~bound :~~~} ( Q_h^k - Q_h^* )( \phi( x,a ) ) \leq & ~ \alpha_t^0 \cdot H + \beta_t + \epsilon + \sum_{i=1} ^t \alpha_t^i \cdot ( V_{h+1}^{k_i} - V_{h+1}^* ) ( x_{h+1}^{k_i} ) ; \label{eq:upper_bound_Q_minus_Q_*} \\
\mathrm{Lower~bound :~~~} ( Q_h^k - Q_h^* )( \phi( x,a ) ) \geq & ~ -2(H-h+1)\epsilon \label{eq:lower_bound_Q_minus_Q_*}
\end{align}
where $t = n_{h,k}(\phi(x,a))$ and $k_1, \cdots, k_t < k$ are the episodes where $\phi(x,a)$ was encountered at step $h$.
\label{lemma:upper_lower_Q}
\end{lemma}

The second result in the above lemma shows that $Q_h^k$ approximately upper bounds $Q_h^*$ at points in the $\epsilon$-net. Via the assumption~\ref{assum:lipschitz}, we can easily extend the above upper bound and lower bound on the difference of $Q_h^k$ and $Q_h^*$ measured at the net's points, to the upper bound and the lower bound of $V_h^* - V_h^*$ measured over the entire state space $\cal S$.

\begin{lemma}
\label{lemma:V_upper_lower}
Following the same setting as in Lemma~\ref{lemma:upper_lower_Q}, for any $(x,h, k)$, with probability at least $1-p$, we have:
\begin{align*}
   \mathrm{Upper~bound:~~~} & V_h^k(x) - V_h^*(x) \leq \alpha_t^0 \cdot H + \beta_t + 2 \epsilon + \sum_{i=1} ^t \alpha_t^i \cdot ( V_{h+1}^{k_i} - V_{h+1}^* ) ( x_{h+1}^{k_i} ) ;\\
   \mathrm{Lower~bound:~~~} &V_h^k(x) - V_h^*(x) \geq -2(H - h + 1.5)\epsilon.
\end{align*}
\end{lemma}
The proof of the above lemma uses Lemma~\ref{lemma:upper_lower_Q}, the definition of $V^*$, and the formula of $V_h$ (Line~\ref{line:update_V} in Alg.~\ref{alg:NBQL}), and assumption~\ref{assum:lipschitz}. We defer the proof to Appendix~\ref{app:v_upper_lower}. 

With these lemmas in hand, we are ready to provide a proof sketch for the main result (Theorem~\ref{thm:main}).

\subsection{Regret Analysis}


We first define $\delta_h^k$ and $\phi_h^k$,
\begin{align*}
\delta_h^k = ( V_h^k - V_h^{\pi_k} ) ( x_h^k ), \text{~and~} \phi_h^k = (V_h^k - V_h^*) (x_h^k) .
\end{align*}
By Lemma~\ref{lemma:V_upper_lower} (lower bound part), we have that with $1-p$ probability 
\begin{align*}
V_h^k(x) \geq V_h^*(x) - 2(H+1) \epsilon, 
\end{align*} for any $h\geq 1$ and $x\in\cal S$. Thus, the total regret can be upper bounded:
\begin{align}\label{eq:total_regret}
\textsc{Regret}(K) = & ~ \sum_{k=1}^K ( V_1^* - V_1^{\pi_k} ) ( x_1^k ) 
\leq  2(H+1)K \epsilon + \sum_{k=1}^K ( V_1^k - V_1^{\pi_k} ) ( x_1^k ) \notag \\
= & ~ 2(H+1)K \epsilon + \sum_{k=1}^K \delta_1^k.
\end{align}

The main idea of the rest of the proof is to upper bound $\sum_{k=1}^K \delta_h^k$ by the next step $\sum_{k=1}^K \delta_{h+1}^k$, thus giving a recursive formula to calculate total regret. We can obtain such a recursive formula by relating $\sum_{k=1}^K \delta_h^k$ to $\sum_{k=1}^K \phi_h^k$.


Recall that $\beta_t = 2 \sum_{i=1}^t \alpha_t^i b_i \leq O(1) \sqrt{  H^3 \gamma / t }$ and $\xi_{h+1}^k = [ ( \P_h - \wh{\P}_h^k ) ( V_{h+1}^* - V_{h+1}^k ) ]$.

Then we have:
\begin{align}\label{eq:upper_bound_delta_h_k}
\delta_h^k 
= & ~ (V_h^k - V_h^{\pi_k}) (x_h^k) \notag \\
\leq & ~ Q_h^k(\phi(x_h^k, a_h^k)) - Q_h^{\pi_k}(x_h^k, a_h^k) \notag \\
= & ~ \left(Q_h^k(\phi(x_h^k,a_h^k)) - Q_h^*(x_h^k,a_h^k)\right) + \left(Q^*_h(x_h^k,a_h^k) - Q_h^{\pi_k}(x_h^k,a_h^k) \right) \notag \\
\leq & ~ ( Q_h^k - Q_h^* ) (\phi( x_h^k , a_h^k) ) + (Q_h^* - Q_h^{\pi_k}) ( x_h^k, a_h^k ) + \epsilon \notag \\
= & ~ ( Q_h^k - Q_h^* ) ( \phi(x_h^k , a_h^k) ) + [ \P_h ( V_{h+1}^* - V_{h+1}^{\pi_k} ) ] ( x_h^k , a_h^k ) + \epsilon \notag \\
\leq & ~ \alpha_{n_{h,k}}^0 H + \beta_{n_{h,k}} + 2\epsilon + \left( \sum_{i=1}^{n_{h,k}} \alpha_{n_{h,k}}^i \cdot \phi_{h+1}^{k_i} \right) +  [ \P_h ( V_{h+1}^* - V_{h+1}^{\pi_k} ) ] ( x_h^k , a_h^k ) \notag \\
= & ~ \underbrace{ \alpha_{n_{h,k}}^0 H }_{ C_1 } + \beta_{n_{h,k}} +2\epsilon + \underbrace{ \left( \sum_{i=1}^{n_{h,k}} \alpha_{n_{h,k}}^i \cdot \phi_{h+1}^{k_i} \right) }_{ C_2 } - \phi_{h+1}^k + \delta_{h+1}^k + \xi_{h+1}^k,
\end{align}
where the first inequality follows from our definition of $V_h^k$ and $V_h^k ( x_h^k ) \leq \max_{ a' \in {\cal A} } Q_h^k ( \phi(x_h^k , a') ) = Q_h^k ( \phi(x_h^k , a_h^k) ) $; the second inequality follows from assumption~\ref{assum:lipschitz}; the third inequality follows from Lemma~\ref{lemma:upper_lower_Q} (upper bound part);  the third equality uses Bellman equation; the last equality follows from definition $\delta_h^{k+1} - \phi_{h+1}^{k} = ( V_{h+1}^* - V_{h+1}^{\pi_k}  ) ( x_{h+1}^k ) $.

We turn to computing the summation $\sum_{k=1}^K \delta_h^k$. Denoting by $n_{h,k} = n_{h,k} ( \phi ( x_h^k , a_h^k ) )$, we can handle the $C_1$ term in Eq.~\eqref{eq:upper_bound_delta_h_k} in the following sense:
\begin{align}\label{eq:upper_bound_delta_h_k_C_1}
\sum_{k=1}^K \alpha_{ n_{h,k} }^0 \cdot H = \sum_{k=1}^K H \cdot \I [ n_{h,k} = 0 ] \leq N H, 
\end{align} 
where $N = |\N_{\epsilon}|$.

The key step is to upper bound the term $C_2$ in Eq.~\eqref{eq:upper_bound_delta_h_k}, which is
\begin{align*}
\sum_{k=1}^K \sum_{i=1}^{ n_{h,k} } \alpha_{ n_{h,k} }^i \phi_{h+1}^{k_i( x_h^k , a_h^k )},
\end{align*}
where $k_i(x_h^k , a_h^k)$ is the episode in which $\phi(x_h^k , a_h^k)$ was taken at step $h$ for the $i$-th time. We regroup the summation in a different way. For every $k' \in [K]$, the term $\phi_{h+1}^{k'}$ appears in the summation with $k > k'$ if and only if $\phi(x_h^k, s_h^k) = \phi( x_h^{k'} , s_h^{k'} )$. The first time it appears we have $n_{h,k} = n_{h,k'} + 1$, the second time it appears we have $n_{h,k} + n_{h,k'} + 2$, and so on. Therefore
\begin{align}\label{eq:upper_bound_delta_h_k_C_2}
\sum_{k=1}^K \sum_{i=1}^{n_{h,k}} \alpha_{ n_{h,k} }^i \phi_{h+1}^{ k_i( x_h^k , a_h^k ) } 
\leq & ~ \sum_{k'=1}^K \phi_{h+1}^{k'} \sum_{t = n_{h,k'} + 1}^{\infty} \alpha_t^{n_{h,k'}} \leq (1+\frac{1}{H}) \sum_{k=1}^K \phi_{h+1}^k ,
\end{align} by the fact that $\sum_{t=i}^{\infty} \alpha_t^i = 1 + \frac{1}{H}$.

Plugging Eq.~\eqref{eq:upper_bound_delta_h_k_C_1} and \eqref{eq:upper_bound_delta_h_k_C_2} back into summation of Eq.~\eqref{eq:upper_bound_delta_h_k} over $k \in [K]$, we have
\begin{align*}
\sum_{k=1}^K \delta_h^k 
\leq & ~ N H + 2K \epsilon + (1+\frac{1}{H}) \sum_{k=1}^K \phi_{h+1}^k - \sum_{k=1}^K \phi_{h+1}^k + \sum_{k=1}^K \delta_{h+1}^k + \sum_{k=1}^K (\beta_{n_{h,k}} + \xi_{h+1}^k )  \\
= & ~ N H + 2K \epsilon + \frac{1}{H} \sum_{k=1}^K \phi_{h+1}^k + \sum_{k=1}^K \delta_{h+1}^k + \sum_{k=1}^K (\beta_{n_{h,k}} + \xi_{h+1}^k) \\
\leq & ~ N H + 2K \epsilon + (1+\frac{1}{H}) \sum_{k=1}^K \delta_{h+1}^k + \sum_{k=1}^K( \beta_{n_{h,k}} + \xi_{h+1}^k) ,
\end{align*}
where the last step follows from $\phi_{h+1}^k \leq \delta_{h+1}^k$ due to the fact that $V_h^{\pi_k} \leq V_h^*$ for any $h$.

Recursing the result for $h \in [H]$, and using the fact $\delta_{H+1}^K = 0$, we have:
\begin{align}\label{eq:upper_bound_delta_h_k_1}
\sum_{k=1}^K \delta_1^k \leq O \left( H^2 N + H K \epsilon + \sum_{h=1}^H \sum_{k=1}^K (\beta_{n_{h,k}} + \xi_{h+1}^k) \right).
\end{align}

We first show how to bound $\beta$ in Eq.~\eqref{eq:upper_bound_delta_h_k_1}
\begin{claim}[bounding $\beta$]\label{cla:upper_bound_beta}
We have
\begin{align*}
\sum_{h=1}^H \sum_{k=1}^K \beta_{n_{h,k}} \leq O( \sqrt{ H^4 N T \gamma } ).
\end{align*}
\end{claim}
\begin{proof}

Recall that $n_{h,k} = n_{h,k} ( \phi( x_h^k, a_h^k ) )$. Using pigeonhole principle, for any $h \in [H]$:
\begin{align*}
\sum_{k=1}^K \beta_{n_{h,k}}
\leq & ~ O(1) \cdot \sum_{k=1}^K \left( \frac{ H^3 \gamma }{ n_{h,k} } \right)^{1/2}  \leq ~ O(1) \cdot \sum_{\phi \in {\cal N}_{\epsilon}} \sum_{n=1}^{ n_{h,K}(\phi) } \left( \frac{ H^3 \gamma }{ n } \right)^{1/2}  \\
\leq &~ O( \sqrt{ H^3 N K \gamma } ) 
=  ~ O( \sqrt{ H^2 N T \gamma } ), & \text{~by~} T = H K
\end{align*}
where the third step follows form $\sum_{\phi \in {\cal N}_{\epsilon}} n_{h,K}( \phi)  = K$ and the LHS of third step is maximized when $n_{h,K}( \phi ) = K / N$ for all $\phi \in {\cal N}_{\epsilon}$.
Thus, we have
\begin{align*}
\sum_{h=1}^H \sum_{k=1}^K \beta_{n_{h,k}} \leq O( \sqrt{ H^4 N T \gamma } ). 
\end{align*}
\end{proof}

Next, we show how to bound $\xi$ in Eq.~\eqref{eq:upper_bound_delta_h_k_1}

\begin{claim}[bounding $\xi$]\label{cla:upper_bound_xi}
We have
\begin{align*}
\left| \sum_{h=1}^H \sum_{k=1}^K \xi_{h+1}^k \right| \leq c H \sqrt{T \gamma }.
\end{align*}
\end{claim}

\begin{proof}

Also, by the Azuma-Hoeffding inequality, with probability $1-p$, we have:
\begin{align*}
\left| \sum_{h=1}^H \sum_{k=1}^K \xi_{h+1}^k \right| 
= & ~ \left| \sum_{h=1}^H \sum_{k=1}^K [ ( \P_h^k - \wh{\P}_h^k ) ( V_{h+1}^* - V_{h+1}^k ) ] ( x_h^k, a_h^k ) \right| \\
\leq & ~ c H \sqrt{T \gamma }.
\end{align*}
\end{proof} Using Eq.~\eqref{eq:upper_bound_delta_h_k_1}, Claim~\ref{cla:upper_bound_beta} and \ref{cla:upper_bound_xi}, we have 
\begin{align*}
\sum_{k=1}^K \delta_1^k \leq O( H^2 N + \sqrt{ H^4 N T \gamma } + \epsilon H K ).
\end{align*}
Next, we can show regret bound. 
If $T \geq \sqrt{ H^4 N T \gamma }$, we have 
\begin{align*}
H^2 N \leq \sqrt{ H^4 N T \gamma }.
\end{align*}
On the other hand if $T \leq \sqrt{H^4 N T \gamma}$, we have
\begin{align*}
\sum_{k=1}^K \delta_1^k \leq HK = T \leq \sqrt{ H^4 N T \gamma },
\end{align*} due to the fact that $\delta_1^k(x) \leq H$ for any $x$. 
Therefore, combining the above two cases gives
\begin{align*}
\textsc{Regret} \leq 3HK\epsilon + \sum_{k=1}^K \delta_1^k \leq O( \sqrt{ H^4 N T \gamma } + \epsilon T )
\end{align*} Recall the definition of the covering dimension. If the $\epsilon$-net is an optimal $\epsilon$-net with $N\leq \epsilon^{-d}$,  set $\epsilon = T^{-1/(d+2)}$, and plug it into the above expression, we have:
\begin{align*}
    \textsc{Regret} \leq \tilde{O}\left(H^2 T^{\frac{d+1}{d+2}}\right).
\end{align*} Hence we prove the theorem.

\section{Conclusion}
In this work, we considered efficient model-free Reinforcement Learning in metric spaces. Under the assumption that the optimal Q function is Lipschitz continuous---hence relaxing the previous assumptions of transition and reward being Lipschitz continuous, we designed \textsc{NbQl}, a Q-learning like algorithm that can achieve a regret bound in the order of $\tilde{O}(T^{(1+d)/(2+d)})$ with $d$ being the covering dimension of the underlying state-action metric space. Unlike previous model-based approaches, our algorithm does not need a planning oracle. Future work includes improving the dependency on the horizon via using Bernstein-type concentration inequality rather than Hoeffding inequality.

\section*{Acknowledgements}
The authors would like to thank Zeyuan Allen-Zhu for useful discussions.

\clearpage
\newpage

\bibliographystyle{alpha}
\bibliography{ref}

\newpage
\appendix

\newpage
\section*{Appendix}

\section{Omitted Proofs}
\subsection{Proof of Proposition~\ref{prop:lipschitz}}
\label{app:lipschitz}
\begin{proof}

For any $h \in [H]$, we have:
\begin{align*}
     & ~ Q^*_h(x,a) - Q^*_h(x',a') \\
    = & ~ r_h(x,a) - r_h(x',a') \\
    & ~ + \int_{x''} P(x''|x,a) Q_{h+1}^*(x'', \pi^*(x'')) \mathrm{d} x''- \int_{x'} P(x''|x',a')Q_{h+1}^*(x'',\pi^*(x'')) \mathrm{d} x'' \\
    \leq & ~ D((x,a), (x',a')) + \|P(\cdot|x,a) - P(\cdot|x',a')\|_1 \cdot \|Q_{h+1}^*\|_{\infty} \\
    \leq & ~ D((x,a), (x',a')) + (H-h) \cdot D((x,a), (x',a'))\\
    = & ~ (H-h+1) \cdot D((x,a), (x',a')),
\end{align*} where the first equality uses Bellman equation, the first inequality uses Lipschitz continuous assumption on $r_h$, and Holder inequality, the second inequality comes from the fact that $Q_{h+1}^*(x,a) \leq H-h$ for any $(x,a)$.

\end{proof}

\subsection{Proof of Lemma~\ref{lemma:q_formula}}
\label{app:q_formula}

\begin{proof}
We fix $\phi(x,a)$. By definition, we have $\phi(x_h^{k_i}, a_{h}^{k_i}) = \phi(x,a)$ for all $i\in [n_{h,k}]$. For notation simplicity, denote $\phi = \phi(x,a)$ and $t = n_{h,k}$. We have:
\begin{align*}
     & ~ Q^{k}_h(\phi) \\
    = & ~ (1-\alpha_{t}) \cdot Q_{h}^{k_t}(\phi) + \alpha_{t} \cdot \left(r_h(x_h^{k_t},a_h^{k_t}) + V_{h+1}^{k_t}(x_{h+1}^{k_t}) + b_{t}\right) \\
    = & ~ (1-\alpha_{t}) \cdot \left(
    (1-\alpha_{t-1}) \cdot Q_h^{k_{t-1}}(\phi) + \alpha_{t-1}\cdot \left(r_h(x_{h}^{k_{t-1}}, a_{h}^{k_{t-1}}) + V_{h+1}^{k_{t-1}}(x_{h+1}^{k_{t-1}}) + b_{t-1} \right) 
    \right)  \\
    & ~ + \alpha_{t} \cdot \left(r_h(x_h^{k_t},a_h^{k_t}) + V_{h+1}^{k_t}(x_{h+1}^{k_t}) + b_{t}\right) \\
    = & ~ \dots\\
    = & ~ \prod_{i=1}^t (1-\alpha_{i}) H + \sum_{i=1}^t \alpha_{i}\prod_{j=i+1}^t (1-\alpha_{j})\left( 
    r_h(x_h^{k_i},a_h^{k_i}) + V_{h+1}^{k_{i}}(x_{h+1}^{k_i}) + b_{i}
    \right) \\
    = & ~ \alpha_{ t }^0 \cdot H + \sum_{i=1}^{ t } \alpha_t^i \cdot \left( r_h(x_h^{k_i},a_h^{k_i}) + V_{h+1}^{k_i}(x_{h+1}^{k_i}) +b_i \right)
\end{align*}
where the first step follows from the update rule of $Q$ in Line~\ref{line:update_Q}, the second step follows from the update rule for $Q_h^{k_{t-1}}$,  the third step follows from recursively representing $Q_h^{k_t}$ using $Q_h^{k_{t-1}}$ until $t=1$, and the last step follows from the definition of $a_t^0$ and $a_{t}^i$.
\end{proof}

\subsection{Proof of Lemma~\ref{lemma:Q_progress}}

Before diving into the detailed proof, we provide some notations below. For any $V:{\cal S}\to \R$, we define $[\P_h V]: {\cal S}\times{\cal A} \to \R$ as $[\P_h V](x,a) \triangleq \E_{x'\sim \P_h(\cdot|x,a)} [V(x')]$. At $k$-th episode and time step $h$, we define $[\hat{\P}_h^{k} V](x_h^k, a_h^k) \triangleq V(x_{h+1}^{k})$.

\begin{proof}

Since $\sum_{i=0}^t \alpha_t^i = 1$, we have that $Q^*_h(x,a) = \alpha_t^0 Q_h^*(x,a) + \sum_{i=1}^t \alpha_t^i Q_h^*(x,a)$.

By assumption~\ref{assum:lipschitz} and the fact that $\N_{\epsilon}$ is an $\epsilon$-net, we have:
\begin{align}
\label{eq:net_approx}
|Q_h^*(x,a) - Q_h^*(\phi(x,a))| \leq D((x,a), (\phi(x,a))) \leq \epsilon.
\end{align} for any $x,a,h$. 
Hence: 
\begin{align*}
Q^*_h(\phi(x,a)) = \alpha_t^0 \cdot Q_h^*(\phi(x,a)) + \sum_{i=1}^t \alpha_t^i \cdot Q_h^*(\phi(x,a)) = \alpha_t^0 \cdot Q_h^*(\phi(x,a)) + \sum_{i=1}^t \alpha_t^i \cdot Q_h^*(\phi(x_h^{k_i},a_h^{k_i})) .
\end{align*}
where the last step follows from $\phi(x,a) = \phi(x_h^{k_i} , a_h^{k_i})$, $\forall i \in [t]$.

Then we have 
\begin{align}
\label{eq:Q_range}
Q^*_h(\phi(x,a)) \in \left[ \alpha_t^0 Q_h^*(\phi(x,a)) + \sum_{i=1}^t \alpha_t^i\left( Q_h^*(x_h^{k_i},a_h^{k_i}) - \epsilon\right), \;\;  \alpha_t^0 Q_h^*(\phi(x,a)) + \sum_{i=1}^t \alpha_t^i \left(Q_h^*(x_h^{k_i},a_h^{k_i}) + \epsilon\right)\right],
\end{align}where we used Inequality~\eqref{eq:net_approx} above.

Now for $Q_h^*(x_h^{k_i}, a_h^{k_i})$, by Bellman equation, we have $Q_h^*(x_h^{k_i}, a_h^{k_i}) = r_h(x_h^{k_i},a_h^{k_i}) + [\P_h V_{h+1}^*](x_h^{k_i},a_h^{k_i})$. Recall $[\hat{\P}_h^{k_i} V_{h+1}](x_h^{k_i},a_h^{k_i} ) = V_{h+1}(x_{h+1}^{k_i})$, we have:
\begin{align*}
    Q_h^*(x_h^{k_i},a_h^{k_i}) = r_h(x_h^{k_i},a_h^{k_i}) + [(\P_h - \hat{\P}_h^{k_i}) V_{h+1}^*] (x_h^{k_i}, a_h^{k_i}) + V_{h+1}^*(x_{h+1}^{k_i}).
\end{align*} 
Substitute the above equality into Eq.~\eqref{eq:Q_range}, we have:
\begin{align*}
    Q^*_h(\phi(x,a)) = \alpha_t^0 Q_h^*(\phi(x,a)) + \sum_{i=1}^t \alpha_t^i \left( r_h(x_h^{k_i},a_h^{k_i}) + [ (\P_h - \hat{\P}_h^{k_i}) V_{h+1}^* ] (x_h^{k_i}, a_h^{k_i}) + V_{h+1}^*(x_{h+1}^{k_i})\pm \epsilon \right) 
\end{align*}

Subtracting the formula in Lemma~\ref{lemma:q_formula} from this equation, we have:
\begin{align*}
     & ~Q_h^k(\phi(x,a)) - Q^*_h(\phi(x,a)) \\
     = & ~ \alpha_t^0(H - Q^*_h(\phi(x,a))) + \sum_{i=1}^t \alpha_t^i \left( (V_{h+1}^{k_i} - V_{h+1}^*)(x_{h+1}^{k_i}) + [ (\hat{\P}_h^{k_i} - \P_h)V_{h+1}^* ] (x_h^{k_i},a_h^{k_i}) + b_i \pm \epsilon
    \right) .
\end{align*}
\end{proof}

\subsection{Proof of Lemma~\ref{lemma:upper_lower_Q}}
\label{app:upper_lower}.

\begin{proof}

{\bf Upper bound $( Q_h^k - Q_h^* )( \phi( x,a ) )$.}

For each fixed $(\phi(x,a),h) \in \N_{\epsilon} \times [H]$, 
 let us define $k_i$ as the episode of which $\phi(x,a)$ was encountered as step $h$ for the $i$-th time. 

Denote $\E_{i}$ as the conditional expectation conditioned on all information till step $h$ episode $k_i$. Then, we have:
\begin{align*}
    \E_{i}\left[ [(\hat{\P}_h^{k_i} - \P) V_{h+1}^*](x,a)   \right] = 0.
\end{align*}
Hence, 
\begin{align*}
\Big\{ [ ( \wh{\P}_h^{k_i} - \P ) V_{h+1}^* ] (x,a) \Big\}_{i=1}^{\tau} 
\end{align*}
is a martingale difference sequence. 

Note $N = |\N_{\epsilon}|$. 
By Azuma-Hoeffding and a union bound over all $K$ episodes, we have that with probability at least $1- p / ( N H )$
\begin{align*}
\forall \tau \in [K] : \left| \sum_{i=1}^{\tau} \alpha_{\tau}^i \cdot [ ( \wh{\P}_h^{k_i} - \P_h ) V_{h+1}^* ] (x,a) \right| 
\leq & ~ \frac{ c H }{ 2 } \left( \sum_{i=1}^{\tau} ( \alpha_{\tau}^i )^2 \cdot \gamma \right)^{1/2} \\
\leq & ~ c \cdot \sqrt{ \frac{ H^3 \gamma }{ \tau } }
\end{align*}
for some absolute constant $c$, with $\gamma = \log(NT/p)$. To get the above result, we used the fact that $V_{h}^*(x) \leq H$ for any $h,x$, and $\sum_{i=1}^{\tau} (\alpha_{\tau}^i)^2 \leq 2H/\tau$ for any $\tau \geq 1$ (see Lemma 4.1 in \cite{jabj18}).

Because inequality holds for all fixed $\tau \in [K]$ uniformly, it also holds for $\tau = t = n_{h,k}(\phi(x,a)) \leq K$. 
Putting it all together, and using a union bound over the points in the $\epsilon$-net $\N_{\epsilon}$ and all time steps, we see that with at least $1-p$ probability, the following holds : for all $\phi(x,a),h,k \in {\cal N}_{\epsilon} \times  [H] \times [K]$:
\begin{align*}
\left| \sum_{i=1}^t \alpha_t^i \cdot [ ( \wh{\P}_h^{k_i} - \P_h ) V_{h+1}^* ] (x,a) \right| \leq c \sqrt{ H^3 \gamma / t }, \text{~~~~~where~} t = n_{h,k} ( \phi( x,a) )
\end{align*} where $\gamma = \log( N T / p )$. 

On the other hand, if we choose $b_t = c \sqrt{ H^3 \gamma /t }$ for the same constant $c$ in the equation above, then we have
\begin{align}\label{eq:def_beta_t}
\beta_t / 2 = \sum_{i=1}^t \alpha_t^i b_i \in [ c \sqrt{ H^3 \gamma /t } , 2 c \sqrt{ H^3 \gamma /t } ],
\end{align} where we used the fact that $1/\sqrt{t} \leq \sum_{i=1}^t \alpha_t^i / \sqrt{i} \leq 2/\sqrt{t}$ (see Lemma 4.1 in \cite{jabj18}).

We have
\begin{align*}
    & ~ (Q_h^k - Q^*_h)(\phi(x,a)) \\
    \leq & ~ \alpha_t^0 \cdot ( H - Q_h^*(\phi(x,a)) ) + \sum_{i=1}^t \alpha_t^i \cdot \left( ( V_{h+1}^{k_i} - V_{h+1}^{*} ) ( x^{k_i}_{h+1} ) + [ ( \wh{\P}_h^{k_i} - \P_h ) V_{h+1}^* ] (x_h^{k_i},a_h^{k_i}) + b_i + \epsilon \right) \\
    \leq & ~ \alpha_t^0(H - Q_h^*(\phi(x,a))) + \sum_{i=1}^t \alpha_t^i ( (V_{h+1}^{k_i} - V_{h+1}^*)(x_{h+1}^{k_i}) + b_i + \epsilon ) + c\sqrt{H^3\gamma/t}\\
    \leq & ~ \alpha_t^0 H + \beta_t + \epsilon + \sum_{i=1}^t \alpha_t^i (V_{h+1}^{k_i} - V_{h+1}^*)(x_{h+1}^{k_i}),
\end{align*}
where the first step follows from Lemma~\ref{lemma:Q_progress}, and the last step follows from $\sum_{i=1}^t \alpha_t^i \leq 1$.

Thus, we complete the proof of Eq.~\eqref{eq:upper_bound_Q_minus_Q_*}.

{\bf Lower bound $( Q_h^k - Q_h^* )( \phi( x,a ) )$.}

To prove Eq.~\eqref{eq:lower_bound_Q_minus_Q_*}, we use induction. 
By definition, we have $Q_{H+1}^k = Q_{H+1}^* = 0$, which implies that $Q_{H+1}^k(\phi(x,a)) - Q_{H+1}^k(\phi(x,a))  = 0 = -2(H-(H+1)+1)\epsilon$. Assume that $Q_{h+1}^k(\phi(x,a)) \geq Q_{h+1}^*(\phi(x,a)) -2(H - (h+1) + 1)\epsilon$ for any $x,a$.  

First, for $V_{h+1}^{k}(x_{h+1}^k) - V_{h+1}^*(x_{h+1}^k)$, we have:
\begin{align*}
    & ~ (V_{h+1}^k - V_{h+1}^*)(x_{h+1}^k) \\
    = & ~ \min \left\{ H, \max_{a'\in{\cal A}}Q_{h+1}^k(\phi(x_{h+1}^k,a')) \right\} - \max_{a\in{\cal A}}Q^*_{h+1}(x_{h+1}^k, a)\\
    \geq & ~ \max_{a' \in {\cal A} } Q_{h+1}^k(f(x_{h+1}^k, a')) - \max_{a\in {\cal A}}Q_{h+1}^*(x_{h+1}^k,a) \\
    \geq & ~ Q_{h+1}^k(f(x_{h+1}^k, \pi^*(x_{h+1}^k))) - Q_{h+1}^*(x_{h+1}^k,\pi^*(x_{h+1}^k))   \\
    \geq & ~ Q_{h+1}^k(\phi(x_{h+1}^k, \pi^*(x_{h+1}^k))) -\epsilon - Q_{h+1}^*(\phi(x_{h+1}^k, \pi^*(x_{h+1}^k))) \\
    \geq & ~ -2(H-h)\epsilon -\epsilon,
\end{align*}
where the second step follows from taking the second term as the argmin (otherwise the result trivially follows as $H \geq Q_{h+1}^*(x,a)$), the third step follows from $\arg\max_a Q_{h+1}^* ( x_{h+1}^k , a ) = \pi^* ( x_{h+1}^k )$ (by definition of $\pi^*$), the fourth step follows uses the fact that $\N_{\epsilon}$ is a $\epsilon$-net and the assumption~\ref{assum:lipschitz}, and the last step follows from induction hypothesis.

Now, we have:
\begin{align*}
    & ~ Q_h^k(\phi(x,a)) - Q_h^*(\phi(x,a)) \\
    \geq & ~ \sum_{i=1}^t \alpha_t^i \cdot \left( ( V_{h+1}^{k_i} - V_{h+1}^{*} ) ( x^{k_i}_{h+1} ) + [ ( \wh{\P}_h^{k_i} - \P_h ) V_{h+1}^* ] (x_h^{k_i},a_h^{k_i}) + b_i - \epsilon \right)\\
    \geq & ~ \sum_{i=1}^t \alpha_t^i \cdot \left( (V_{h+1}^{k_i} - V_{h+1}^*)(x_{h+1}^{k_i}) - \epsilon \right) \\
    \geq & ~ - \left( \sum_{i=1}^t \alpha_t^i (2(H-h)\epsilon + \epsilon) \right) - \epsilon \\
    \geq & ~ -2(H-h)\epsilon - 2\epsilon\\
     = & ~ -2(H-h+1)\epsilon,
\end{align*} 
where the first step we used the fact that $H \geq Q^*_h(f(x,a))$, the second step uses the fact that $ \sum_{i=1}^t \alpha_t^i [ ( \wh{\P}_h^{k_i} - \P_h ) V_{h+1}^* ] (x_h^{k_i},a_h^{k_i}) \geq -c\sqrt{H^3\gamma/t} \geq -\sum_{i=1}^t\alpha_t^i b_i$, and the third step uses the result that $(V_{h+1}^{k_i} - V_{h+1}^*)(x_{h+1}^{k_i}) \geq -2(H-h)\epsilon - \epsilon$, where the fourth step uses the fact that $\sum_{i=1}^t \alpha_t^i \leq 1$.

Thus, we prove Eq.~\eqref{eq:lower_bound_Q_minus_Q_*}. We complete the proof.

\end{proof}

\subsection{Proof of Lemma~\ref{lemma:V_upper_lower}}
\label{app:v_upper_lower}
\begin{proof}
For any $h,x$, by the definition of $V_h^k$ from Alg~\ref{alg:NBQL}, we have:
\begin{align*}
    V_h^k(x) - V_h^*(x) 
    = & ~ \min\{H, \max_{a}Q_h^k(\phi(x,a))\} - \max_a Q_h^*(x,a) \\
    = & ~ \max_a Q_h^k(\phi(x,a)) - \max_a Q_h^*(x,a) \\
    \geq & ~ \max_a Q_h^k(\phi(x,a)) - \max_a Q_h^*(\phi(x,a)) - \epsilon\\
    \geq & ~ Q_h^k(\phi(x,a^*)) - Q_h^*(\phi(x,a^*)) - \epsilon \\
    \geq & ~ -2(H - h + 1)\epsilon - \epsilon,
\end{align*} where the second equality takes the min at the second term (otherwise the result follows as $Q_h^*(x,a)\leq H$ for any $(x,a)$), the first inequality comes from the $\epsilon$-net construction, and in the second inequality we denote $a^*=\arg\max_a Q_h^*(\phi(x,a))$, while the last inequality uses the lower bound in Lemma~\ref{lemma:upper_lower_Q}.

For upper bound, we have:
\begin{align*}
    V_h^k(x) - V_h^*(x) 
    = & ~ \min\{H, \max_{a}Q_h^k(\phi(x,a))\} - \max_a Q_h^*(x,a) \\
    \leq & ~ \max_{a}Q_h^k(\phi(x,a)) - \max_a Q_h^*(x,a)\\
    \leq & ~ Q_h^k(\phi(x,a^k)) - Q_h^*(x,a^k) \\
    \leq & ~ Q_h^k(\phi(x,a^k)) - Q_h^*(\phi(x,a^k)) + \epsilon \\
    \leq & ~  \alpha_t^0 \cdot H + \beta_t + 2 \epsilon + \sum_{i=1} ^t \alpha_t^i \cdot ( V_{h+1}^{k_i} - V_{h+1}^* ) ( x_{h+1}^{k_i} ),
\end{align*} where in the first inequality we denote $a^k = \arg\max_a Q_h^k(\phi(x,a))$ and $Q_h^*(x,a^k) \leq \max_a Q_h^*(x,a)$, and in the second inequality, we use $\epsilon$-net construction, and the last inequality uses the upper bound result in lemma~\ref{lemma:upper_lower_Q}. 

\end{proof}







\end{document}